\documentclass[runningheads]{llncs}
\usepackage{graphicx}

\usepackage{url}
\usepackage{hyperref}
\usepackage{amsmath}
\usepackage{booktabs}
\usepackage{xspace}
\usepackage{pifont}
\usepackage{threeparttable}
\usepackage{tikz}
\usepackage{graphicx} 
\usepackage{pgfplots}
\usepackage{multirow}
\RequirePackage{algorithm}
\RequirePackage{algorithmic}
\usepackage{enumitem}
\usepackage{wrapfig}

\newcommand{\bert}{\textsc{BERT}\xspace}
\newcommand{\gnn}{\textsc{GNN}\xspace}
\newcommand{\mlp}{\textsc{MLP}\xspace}
\newcommand\gradbertemoji{\raisebox{-2pt}{\includegraphics[width=0.9em]{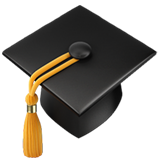}}}
\newcommand{\method}{\textsc{GraDBERT}\xspace}
\newcommand{\grad}{\textsc{GraD}\xspace}

\newcommand{\gradbert}{\textsc{GraD}BERT\xspace}
\newcommand{\gradmlp}{\textsc{GraD}MLP\xspace}

\newcommand{\gradj}{\textsc{GraD}-Joint\xspace}
\newcommand{\grada}{\textsc{GraD}-Alt\xspace}
\newcommand{\gradkd}{\textsc{GraD}-JKD\xspace}
\newcommand{\ngradj}{Joint\xspace}
\newcommand{\ngrada}{Alt\xspace}
\newcommand{\ngradkd}{JKD\xspace}

\newcommand{\graphless}{graph-free\xspace}

\newcommand{\myTag}[1]{ {\bf #1}}

\usepackage{amsmath,amsfonts,bm}

\def\Figref#1{Figure~\ref{#1}}

\def\Secref#1{Section~\ref{#1}}

\def\eqref#1{~(\ref{#1})}
\def\Eqref#1{Eq.(\ref{#1})}

\def\1{\bm{1}}

\def\vh{{\bm{h}}}

\def\vs{{\bm{s}}}
\def\vt{{\bm{t}}}

\def\vx{{\bm{x}}}
\def\vy{{\bm{y}}}

\def\mW{{\bm{W}}}

\def\mY{{\bm{Y}}}

\DeclareMathAlphabet{\mathsfit}{\encodingdefault}{\sfdefault}{m}{sl}
\SetMathAlphabet{\mathsfit}{bold}{\encodingdefault}{\sfdefault}{bx}{n}

\def\gE{{\mathcal{E}}}

\def\gG{{\mathcal{G}}}

\def\gL{{\mathcal{L}}}

\def\gN{{\mathcal{N}}}
\def\gO{{\mathcal{O}}}

\def\gV{{\mathcal{V}}}

\def\gX{{\mathcal{X}}}

\def\sR{{\mathbb{R}}}

\DeclareMathOperator*{\argmax}{arg\,max}

\usepackage[capitalize,noabbrev]{cleveref}

\begin{document}
\title{Train Your Own GNN Teacher: Graph-Aware Distillation on Textual Graphs}
\titlerunning{Graph-Aware Distillation on Textual Graphs}
\author{Costas Mavromatis\inst{1}\thanks{Work done while interning at Amazon Web Services, Santa Clara. Correspondence to: \textit{mavro016@umn.edu}} \and
Vassilis N. Ioannidis\inst{2} \and
Shen Wang\inst{2} \and Da Zheng\inst{2} \and Soji Adeshina\inst{2} \and Jun Ma\inst{2} \and Han Zhao\inst{2,3} \and Christos Faloutsos\inst{2,4} \and George Karypis\inst{1,2}}
\authorrunning{C. Mavromatis et al.}
\institute{University of Minnesota \and Amazon Web Services \and University of Illinois at Urbana-Champaign \and Carnegie Mellon University}
\maketitle              %
\begin{abstract}
How can we learn effective node representations on textual graphs? Graph Neural Networks (GNNs) that use Language Models (LMs) to encode textual information of graphs achieve state-of-the-art performance in many node classification tasks. Yet, combining GNNs with LMs has not been widely explored for practical deployments due to its scalability issues. In this work, we tackle this challenge by developing a Graph-Aware Distillation framework (\grad) to encode graph structures into an LM for \graphless, fast inference. Different from conventional knowledge distillation, \grad jointly optimizes a GNN teacher and a \graphless student over the graph's nodes via a shared LM. This encourages the \graphless student to exploit graph information encoded by the GNN teacher while at the same time, enables the GNN teacher to better leverage textual information from unlabeled nodes. As a result, the teacher and the student models learn from each other to improve their overall performance. Experiments in eight node classification benchmarks in both transductive and inductive settings showcase \grad's superiority over existing distillation approaches for textual graphs. Our code is available at: \url{https://github.com/cmavro/GRAD}.

\end{abstract}
\section{Introduction} \label{sec:intro}

Graph Neural Networks (GNNs) offer state-of-the-art performance on graph learning tasks in real world applications, including social networks, recommendation systems, biological networks and drug interactions. GNNs~\cite{kipf2017gcn,hamilton2017graphsage,velivckovic2017gat} learn node representations via a recursive neighborhood aggregation scheme~\cite{gilmer2017neural}, which takes as input the node features and the graph structure. In textual graphs, text-based information is associated with the graph's nodes. Methods such as  bag-of-words, word2vec~\cite{mikolov2013distributed} or pre-trained Language Models (LMs), e.g., BERT~\cite{devlin2019bert}, are used to transform raw texts of nodes into features. Transforming raw text to numerical features is usually associated with a non-negligible cost. For example, LMs that rely on transformers~\cite{vaswani2017attention} have a very large number of parameters and their cost depends on the LM's architecture\footnote{For example, the inference cost of a single transformer layer is $\gO(L^2d+Ld^2)$, where $L$ is the sequence length and $d$ is the number of hidden dimensions.}. This high computational cost results in expensive and/or slow inference during deployment.

Recent works show that combining GNNs with LMs in an end-to-end manner~\cite{ioannidis2022graph-aware-bert,yang2021graphformers} leads to state-of-the-art performance for node classification and link prediction tasks. Although powerful, these models are associated with expensive inference costs. During mini-batch inference, a $K$-layer GNN  fetches an exponential number of nodes w.r.t. $K$, where raw texts of the fetched nodes need to be transformed to numerical features on the fly.  
As a result, combining LMs with GNNs in a cascaded manner exponentially grows the LM inference cost for each target node. This cost is prohibitive for applications where fast inference is vital, such as providing instant query-product matching in e-commence. 
 
Aiming at a balance between effectiveness and efficiency, we seek to transfer useful knowledge from a GNN teacher to a \graphless student via distillation. As the \graphless student does not use the graph structure during inference, the inference cost depends only on the LM employed. However, existing knowledge distillation (KD) approaches for graphs either do not take full advantage of the graph structure~\cite{zhang2022glnn} or they require powerful student models (e.g., large LMs) to achieve good performance~\cite{zhao2022glem}.

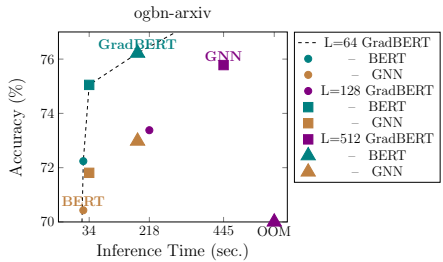
\begin{wrapfigure}{r}{0.5\textwidth}
  \centering
    \resizebox{\linewidth}{!}{\definecolor{col1}{rgb}{0.60, 0.31, 0.64}
\definecolor{col2}{rgb}{0.30, 0.69, 0.29}
\definecolor{col3}{rgb}{0.22, 0.49, 0.72}
\definecolor{col4}{rgb}{0.89, 0.10, 0.11}
\definecolor{col5}{rgb}{1, 1, 0.8}

\begin{tikzpicture}

\tikzstyle{every node}=[font=\large]
\begin{axis}[enlargelimits=false, ylabel={\Large Accuracy (\%)} , xlabel={\Large Inference Time (sec.)}, ymin=70, ymax=77,  legend pos=outer north east, title={\Large ogbn-arxiv}, xtick={34, 218, 445, 600}, xticklabels={34, 218, 445, OOM}, xmin=-60, xmax=640]

\addplot[draw,dashed,thin] 
    coordinates {(12,70) (16, 72.24) (34, 75.05) (182, 76.22) (300,77)};  
    \addplot[
        scatter/classes={
        a={mark=*,mark size = 3pt,teal}, 
        b={mark=*,mark size = 3pt,brown}, 
        c={mark=*,mark size = 3pt,violet},
        d={mark=square*,mark size = 4pt,teal},
        e={mark=square*,mark size = 4pt,brown},
        f={mark=square*,mark size = 4pt,violet},
        g={mark=triangle*,mark size = 7pt,teal},
        h={mark=triangle*,mark size = 7pt,brown},
        i={mark=triangle*,mark size = 7pt,violet}},
        scatter, mark=*, only marks, 
        scatter src=explicit symbolic,
        nodes near coords*={\Label},
        mark size=2pt,
        visualization depends on={value \thisrow{label} \as \Label} %
    ] table [meta=class] {
        x y class label
        16 72.24 a \;
        16 70.43 b {\color{brown}\textbf{BERT}}
        218 73.38 c \;
        34 75.05 d \;
        34 71.81 e \;
        445 75.78 f {\color{violet}\textbf{GNN}}
        182 76.22 g {\color{teal}\textbf{GradBERT}}
        182 72.98 h \;
        600 70 i \;
    };
    
    \legend{L=64 GradBERT, -- \; BERT , -- \; GNN ,
    L=128 GradBERT, -- \; BERT , -- \; GNN ,
    L=512 GradBERT, -- \; BERT , -- \; GNN }

\end{axis}

\end{tikzpicture}}
  \caption{ \myTag{\gradbert} is on the Pareto front: Accuracy performance w.r.t. inference time for ogbn-arxiv. 
    $L$ is the input sequence length and
    OOM (out of memory) means that the model encountered GPU failure. }
    \label{fig:intro}
\end{wrapfigure}

To address these limitations, we developed a Graph-Aware Distillation approach (\grad) that \emph{jointly}
optimizes the GNN teacher with its graph-free student via a shared LM. The shared LM serves as an interaction module that allows the two models to learn from each other. On the one hand, the GNN teacher updates the LM's parameters with graph-aware information and distills graph knowledge as soft-labels, which are provided to the student. On the other hand, the \graphless student imitates the GNN's predictions and leverages textual information from unlabeled nodes to improve the fine-tuning of the shared LM. This dynamic interplay between the two models stimulates the student to not only mimic the GNN predictions but to learn node features that improve its teacher's performance.

\grad is formulated as a multi-task learning for the shared LM, whose goal is to achieve good performance for both the GNN teacher and the \graphless student models. We designed three different strategies for optimizing the \grad framework. Their key differences is on how tight the teacher and student models are coupled and on how much flexibility the student model has to fit to its teacher's predictions. As a result, \grad can be applied in both large-scale graphs, where the student model is a powerful LM, and in small-scale graphs, where the LM is substituted by simple MLPs.

\Figref{fig:intro} illustrates the superiority of \grad when the \graphless student model is a BERT model (\gradbert). It outperforms a fine-tuned BERT model for node classification by 3.24\%, and it is as effective as combining BERT with a GNN with 2.4x-13x smaller inference time.
Our  contributions are summarized below:
\begin{itemize}
    \item We analyze and identify the limitations of conventional knowledge distillation for textual graphs, which have been previously under-studied  (\Secref{sec:mi}).
    
    \item We present a graph-aware distillation (\grad) framework that couples a GNN teacher and a \graphless student together to fully exploit the underlying graph structure. \grad improves  classification for both seen and unseen nodes, that are present in either large or small-scale graphs (\Secref{sec:exp-graphless} and \Secref{sec:exp-ind}). 
    
    \item We developed three different strategies (\Secref{sec:grad}) to effectively optimize \grad framework, which we comprehensively study. This enables \grad to scale to large graphs and achieve state-of-the-art performance in node classification tasks.
\end{itemize}

\section{Background}

\subsection{Problem Formulation} \label{sec:pb}
In an input \emph{textual graph} $\gG = \{ \gV, \gE\}$, each node $v \in \gV$ is associated with raw text, which we denote as $X_v$. $\gV$ is the node set and $\gE$ is the edge set. Let $N$ denote the total number of nodes.
 For node classification, the prediction targets are $ \mY \in \sR^{N \times m}$, where row $\vy_v$ is a $m$-dim one-hot vector for node $v$. The node set is divided into labeled nodes $\gV^L$ and unlabeled nodes  $\gV^U$, i.e., $\gV = \gV^L \cup \gV^U$. In inductive scenarios, the input graph $\gG$ is divided into two subgraphs $\gG= \gG^\text{tran} \cup \gG^\text{ind}$, where $\gG^\text{tran}$ is used for learning (transductive part) and $\gG^\text{ind}$ is used only during inference ($\gG^\text{ind}$ is the inductive part that is not observed during learning). In transductive scenarios, we have $\gG= \gG^\text{tran}$ and $\gG^\text{ind} = \emptyset$.

As introduced \Secref{sec:intro}, utilizing the textual graph $\gG$ for making predictions at test time results in slow inference. Thus, we seek to learn a \graphless model $\tau'$, e.g., a LM, that only takes node text $X_v$ during inference.  It is desired that $\tau'(X_v) \approx f(\gG, X_v)$, where $f$ is a model that uses the graph, e.g., a GNN, so that the \graphless model achieves as effective node classification as a graph-based model.

\subsection{GNNs on Textual Graphs}\label{sec:gnns}

To handle raw texts in textual graphs, $X_v$ is transformed to numerical features $\vx_v \in  \sR^d$ via a function $\tau(\cdot)$,
\begin{equation}
\vx_v = \tau(X_v).
    \label{eq:bert-trans}
\end{equation}
For example, 
LMs, such as BERT~\cite{devlin2019bert} for modelling $\tau(\cdot)$, transform each token of the input sequence $X_v$ to a representation. The final $\vx_v$ can be obtained as the representation of a specific token, which is usually the [CLS] token. 

GNNs~\cite{kipf2017gcn,velivckovic2017gat} transform a computation graph  $\gG_v$ that is centered around node $v$ to a $d$-dimensional node representation $\vh_v$. 
We write the GNN transformation as
\begin{equation}
    \vh_v = \gnn \big( \{\tau(X_v): u \in \gG_v\} \big),
\end{equation} 
where $\tau(X_u)$ generates the input features $\vx_u$ of node $u$.
For a GNN with $K$ layers,  $\gG_v$ includes nodes and edges (with self-loops) up to $K$ hops away from $v$. 

The ($k+1$)-th GNN layer takes as input a node's representation $\vh^{(k)}_v$ of the previous layer $k$ as well as the representations of its 1-hop neighbors. It aggregates them to a new representation $\vh^{(k+1)}_v$ as  follows,
\begin{equation}
    \vh^{(k+1)}_v = \phi \big(  \{\vh^{(k)}_u : u \in \gN_v\}\big), 
    \label{eq:gnn-update}
\end{equation}
where $\gN_v$ is the set of direct neighbors of $v$ and $\phi(\cdot)$ is an aggregation function.
For example, a common GNN update $\phi(\cdot)$, which is employed in GraphSAGE~\cite{hamilton2017graphsage} and RGCN~\cite{schlichtkrull2018rgcn}, can be described as follows,
\begin{equation}
    \vh^{(k+1)}_v = \sigma \Big( \mW^{(k)}_{\text{self}} \vh^{(k)}_v + \sum_{ u \in \gN_v} \mW^{(k)} \vh^{(k)}_u \Big), 
    \label{eq:sage}
\end{equation}
where $\mW^{(k)}_{\text{self}}, \mW^{(k)}$ are learnable parameters and $\sigma(\cdot)$ is a nonlinearity mapping. At the first layer, we usually have $\vh^{(0)}_v = \vx_v$, that are text features extracted from the node.
Computing $\vh_v$ can be as expensive as $\gO(S^K C)$, where $\tau(\cdot)$ is a LM with inference cost $C$, $S$ is the neighborhood size, and $K$ is the number of GNN layers.

\section{Towards Graph-Aware Knowledge Distillation}\label{sec:analysis}
As discussed in \Secref{sec:pb}, we seek to learn a \graphless model $\tau'$ for fast inference, that also achieves as effective node classification as a GNN model $f$.   

\subsection{Knowledge Distillation} \label{sec:kd}
A straightforward solution is to use the Knowledge Distillation (KD) technique~\cite{hinton2015distilling}, where a powerful teacher model transfers knowledge to a simpler student model. In our case, the teacher model corresponds to a graph-based model (GNN), while the student model is a \graphless model, such as an LM.   

We follow the standard KD paradigm, in which the teacher distills knowledge via soft-labels. At a high level, the algorithmic procedure is 
\begin{itemize}
    \item \textbf{First Stage}. The GNN teacher is trained for node classification.
    The objective is given by
    \begin{equation}
        \gL_{\text{nc}} = \sum_{v \in \gV^L} l_{\text{CE}}( \hat{\vt}_v , \vy_v),
    \end{equation}
    where $l_{\text{CE}}$ is the standard cross-entropy and \begin{equation}
        \hat{\vt}_v = \mlp \Big(\gnn \big( \{\vx_u : u \in \gG_v\} \big) \Big)
        \label{eq:gnn-kd}
    \end{equation}
    are teacher's label predictions (logits). Numerical features $\vx_u$ are learned by fine-tuning an LM $\tau(\cdot)$ in an end-to-end manner; see \Eqref{eq:bert-trans}. The trained  GNN teacher generates soft-labels $\hat{\vt}_v$ for all nodes $v \in \gG$. %
    \item \textbf{Second Stage}. Another LM  $\tau'(\cdot)$ is trained to mimic GNN's predictions $\hat{\vt}_v$.
    The LM is optimized via 
    \begin{equation}
        \gL_{\text{KD}} = \sum_{v \in \gV^L} l_{\text{CE}}(\hat{\vs}_v, \vy_v) + \lambda \sum_{v \in \gV} l_{\text{KL}}(\hat{\vs}_v, \hat{\vt}_v),
        \label{eq:kd}
    \end{equation}
    where 
    \begin{equation}
        \hat{\vs}_v = \mlp \big(\tau'(X_v) \big)
    \end{equation}
    are the student's predictions and  $l_{\text{KL}}$ is the KL-divergence between the student's and teacher's logits. Hyper-parameter  $\lambda \in \sR$  controls the relative importance of the knowledge distillation term.
\end{itemize}

\subsection{What Does Knowledge Distillation Learn? An Analysis} \label{sec:mi}
In many cases, only a subset of the nodes in a graph is labeled; it is denoted as $\gV^L$. Since the GNN's predictions  $\hat{\vt}_v$ are treated as soft-labels, the second term in \Eqref{eq:kd} allows KD to fine-tune its LM $\tau'$ over both labeled and unlabeled nodes.
However, as $\hat{\vt}_v$  are pre-computed after the first stage of GNN training, the \graphless student does not use the actual graph structure during learning (second stage). In other words, nodes are treated independently from the underlying graph and the \graphless model can only infer how nodes interact via the provided soft-labels by the GNN. 

Next, we quantify the importance of capturing node interactions, which benefits many applications such as community detection and label propagation in graphs~\cite{jia2020residual}.  As discussed in \Secref{sec:gnns}, the node classification objective aims at transforming a subgraph $\gG_v$ centered around node $v$ to its label $\vy_v$. In information theory, this is equivalent to maximizing the mutual information $I(\cdot)$,
\begin{equation}
    \max_f \sum_{v \in \gV^L} I_f(\vy_v; \gG_v),
\end{equation}
between $\gG_v$ and label distribution $\vy_v$, where $I_f$ is parametrized by the GNN function $f$. It is always true that $I(\cdot) \geq 0$.

If we consider $\gG_v$ as a joint distribution of the text nodal feature set  $\gX_v = \{X_u: u \in \gG_v\}$ and edge set $\gE_v $, we have
\begin{align}
    I(\vy_v; \gG_v) = I(\vy_v ; (\gX_v, \gE_v)) = 
    I(\vy_v; \gX_v ) + I(\vy_v ; \gE_v| \gX_v),
\end{align}
which is obtained via the chain rule of mutual information. Now by setting $\tilde{\gX}_v = \gX_v \setminus \{X_v\}$, we can further obtain
\begin{align}
 I(\vy_v; \gG_v) = 
 I(\vy_v; X_v) + I(\vy_v;\tilde{\gX}_v | X_v)  + I(\vy_v ; \gE_v| \gX_v).
    \label{eq:mi2}
\end{align}
The joint mutual information is decomposed into three terms, (i) the information that comes from the node itself, (ii) the information that comes from other nodes in $\gG_v$, and (iii) the additional information that comes from the actual links between nodes.

\begin{remark} \label{remark1}
If the neighbor set or the graph structure is not utilized during learning, only $I(\vy_v; X_v) $ of \Eqref{eq:mi2} can be maximized. 
\end{remark}

\begin{corollary} \label{corollary}
Assume that predictions $\hat{\vt}_u$ for unlabeled nodes $u \in \gV^U$ are obtained via \Eqref{eq:gnn-kd} with a GNN $f$ and an LM $\tau$. The \graphless student of KD, that employs an LM $\tau'$, solves
\begin{align} \label{eq:kd-frame}
    \max_{\tau'} \quad &  \sum_{u \in \gV^U} I_{\tau'}(\hat{\vt}_u; X_u), 
\end{align}
where
\begin{align} \label{eq:kd-frame2}
\begin{split}
     \quad & \hat{\vt}_u = f \big(  \tau (X_{u'}) : u' \in \gG_u \big), \\
     \textrm{s.t. } \quad  &  f, \tau = \argmax_{\tilde{f}, \tilde{\tau}} \sum_{v \in \gV^L} I_{\tilde{f}, \tilde{\tau}}(\vy_v; \gG_v).
\end{split}
\end{align}
\end{corollary}

The student models solves \Eqref{eq:kd-frame} without directly accessing the graph structure $\gG$.  Since $\hat{\vt}_u$ are obtained through a teacher GNN in \Eqref{eq:kd-frame2}, which maximizes all three parts of \Eqref{eq:mi2}, it is possible that $\hat{\vt}_u$ can still implicitly contain graph information from $\gG_u$. However, due to the data-processing inequality of mutual information, it holds that
\begin{equation}
    I_{f,\tau}(\vy_u; \gG_u) \geq  I(\vy_u; \hat{\vt}_u).
\end{equation}
This means that,  despite being a function of $\gG_u$, the
information contained in $\hat{\vt}_u$ w.r.t. $\vy_u$ is less than that of the original $\gG_u$. Thus, student's performance depends on \emph{how informative} $\hat{\vt}_u$ are, compared to $\gG_u$.

Note that the  objectives of \Eqref{eq:kd-frame} and \Eqref{eq:kd-frame2} are applied over different sets of nodes $\gV^U$ and $\gV^L$, which does not ensure that GNN $f$ is the optimal function to encode graph information for nodes $u \in \gV^U$.
For example, a GNN might  overfit to nodes that appear frequently in the training subgraphs, and ignore nodes that appear infrequently.  
As a result, the quality of predictions $\hat{\vt}_u$ is in question. 
More importantly, it is not clear how the \graphless student  will perform on unseen nodes during training and, for which, soft-labels $\hat{\vt}_u$ cannot be apriori computed by the GNN. However  in textual graphs, nodes are associated with informative textual features (e.g., paper abstracts) and a well-trained student could potentially handle the textual information of the unseen (inductive) nodes. We provide a failing case of an under-trained student in Appendix~\ref{app:coro}.

\begin{figure*}[t]
    \centering
    \includegraphics[width=\linewidth]{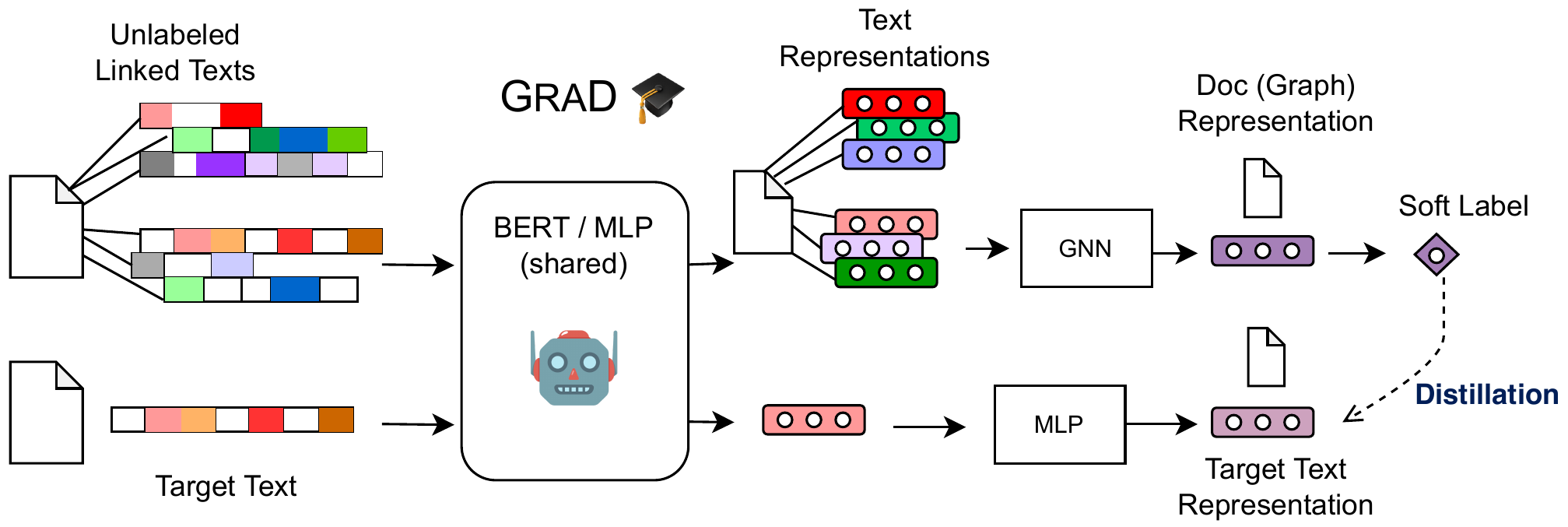}
    \caption{\grad framework. \grad captures textual information among unlabeled linked texts by allowing the teacher GNN and the \graphless student to jointly update the shared text encoding function.}
    \label{fig:framework}
\end{figure*}

\section{\grad \gradbertemoji{} Framework} \label{sec:grad}
In this work, we take a different approach from the one shown in \Eqref{eq:kd-frame}. As we cannot use the graph structure during inference, our \underline{GR}aph-\underline{A}ware \underline{D}istillation (\grad)
improves structure utilization during training. The motivation is to use the original graph $\gG$ while training the \graphless student, even if GNN predictions $\hat{\vt}_u$ may (or may not) implicitly contain such graph information.

Our \grad framework does this by coupling the teacher GNN with its graph-free student with a shared function $\tau(\cdot)$. This allows the GNN to directly encode structural information into $\tau$, that is then used by the \graphless model during inference. The overall framework is illustrated in \Figref{fig:framework}. 
\grad is posed as a multi-task learning for the shared LM, whose goal is to collectively optimize the GNN teacher and the \graphless student models. The learning problem is given by
\begin{align} \label{eq:grad-frame}
\begin{split}
    \max_{f,\tau} \quad & \sum_{v \in \gV^L} I_{f,\tau}(\vy_v; \gG_v) +  \sum_{v \in \gV^U} I_{\tau}(\hat{\vt}_u; X_u), \\
     \textrm{where } \quad & \hat{\vt}_u = f \big(  \tau (X_{u'}) : u' \in \gG_u \big),
\end{split}
\end{align}
and function $\tau$ contributes in both objectives that \grad maximizes. 
This results in a \emph{coupled} multi-objective optimization, since predictions $\hat{\vt}_u$ depend on function $\tau$ and thus, tie the GNN $f$ with the \graphless model together. 

Notably, the graph-free student uses soft-labels provided by the GNN to fine-tune the LM over textual information from unlabeled nodes; see second term of \Eqref{eq:grad-frame}. This, consequently, ensures that the GNN leverages all the graph's textual information (in an implicit manner) to compute $\hat{\vt}_u$, which might have been neglected in \Eqref{eq:kd-frame2}. As  the GNN teacher better encodes graph-aware textual information from the input (transdutive) nodes,  the \graphless student manages to \emph{train its own GNN teacher} via the shared LM. 

In what follows, we present three different optimization strategies (\gradj, \grada, and \gradkd)  for solving \Eqref{eq:grad-frame}. \Figref{fig:grad_strat} presents them at a high-level, and their key difference is on how tight they couple the teacher and the student model.

\begin{figure*}[t]
    \centering
    \includegraphics[width=\linewidth]{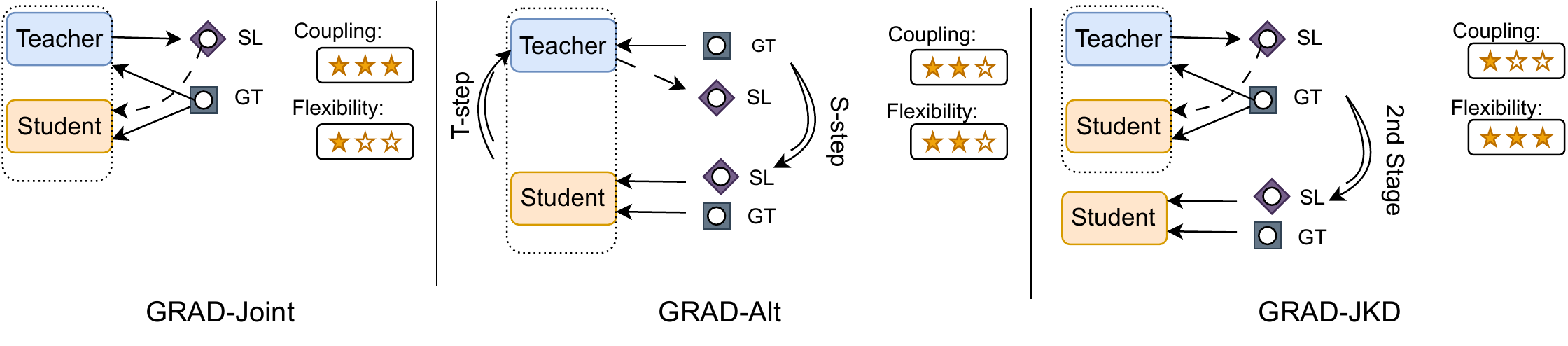}
    \caption{\grad strategies for coupling the GNN teacher and the \graphless student. \textit{SL} denotes soft-label and \textit{GT} denotes ground-truth label. \gradj couples the two models in a single step, \grada couples them in an alternate fashion, while \gradkd couples them in the first stage and decouples them in the second stage. For illustration purposes only, the star metric ($\star$) quantifies the teacher-student coupling tightness and the flexibility of the student model.}
    \label{fig:grad_strat}
\end{figure*}

\subsection{\gradj} \label{sec:gradj}
The first strategy (\gradj) optimizes \grad framework in a single-stage, where the teacher and the student are jointly updating $\tau$. 
\gradj objective is given by
\begin{align}\label{eq:grad}
\begin{split}
    \gL_{\text{Joint}}  =  \lambda \sum_{v \in \gV} l_{\text{KL}}(\hat{\vs}_v, \hat{\vt}_v) +  \sum_{v \in \gV^L} \big[ \alpha l_{\text{CE}}(\hat{\vt}_v, \vy_v) + (1-\alpha) 
      l_{\text{CE}}(\hat{\vs}_v, \vy_v) \big]  , 
\end{split}
\end{align}
where $\lambda \in \sR$ and $\alpha \in [0,1]$ are hyperparameters that control the importance of each term and ensure that the GNN models does not learn to ignore the structure.  The first term is the KD loss between GNN teacher's predictions $\hat{\vt}_v$ and \graphless student's predictions $\hat{\vs}_v$, and the other two terms are for ground-truth labels $\vy_v$. 
Predictions $\hat{\vs}_v$ and $\hat{\vt}_v$ are optimized via a shared encoding function $\tau(\cdot)$ (multi-task learning).

We highlight the dynamic interplay between the GNN teacher and the \graphless student in \Eqref{eq:grad}. The GNN teacher provides a soft-label $\hat{\vt}_v$, which the \graphless student tries to mimic via $l_{\text{KL}}(\hat{\vs}_v, \hat{\vt}_v)$. However, instead of training $\tau(\cdot)$ to simply mimic the GNN predictions, term $l_{\text{CE}}(\hat{\vt}_v, \vy_v)$ encourages $\tau(\cdot)$ to generate text representations that are beneficial for the GNN teacher. This co-learning between the GNN teacher and the \graphless student leads to graph-aware text representations, as the two models learn from each other.

\subsection{\grada} \label{sec:grada}
 A common challenge in multi-task learning is negative transfer~\cite{zhang2020surveynegative}, in which the performance improvement of a specific task leads to performance degradation of other tasks. 
Our second strategy (\grada) alleviates this issue, which might be present in \gradj, by optimizing the teacher and the student model in an alternate fashion. \grada objective is given by
\begin{align}
   & \gL_{\text{T-step}}  = \sum_{v \in \gV^L}  l_{\text{CE}}(\hat{\vt}_v, \vy_v),\label{eq:estep}\\
   & \gL_{\text{S-step}}  =  \lambda \sum_{v \in \gV} l_{\text{KL}}(\hat{\vs}_v, \hat{\vt}_v) +  \sum_{v \in \gV^L}   l_{\text{CE}}(\hat{\vs}_v, \vy_v), 
    \label{eq:mstep}
\end{align}
where the teacher objective $\gL_{\text{T-step}}$ and the student objective $\gL_{\text{S-step}}$ alternately update function $\tau$. Different from \gradj, decoupling S-step from the teacher gives more flexibility to the student to fit to its teacher's predictions.

\subsection{\gradkd} \label{sec:gradkd}

Solving \Eqref{eq:estep} and \Eqref{eq:mstep} consists of many alternate optimization steps that update $\tau$. This is associated with significant computational costs when $\tau$ is a large LM and when the input graph consists of a large number of nodes.

Our third strategy (\gradkd) is motivated by the benefits discussed in \Secref{sec:gradj} that lead to improving the GNN teacher as well as the benefits discussed in \Secref{sec:grada}  that lead to improving the flexibility of the \graphless student. \gradkd consists of two distinct stages, as summarized below
\begin{itemize}
    \item \textbf{First Stage}. The GNN teacher is jointly trained with its \graphless student via objective $\gL_{\text{Joint}}$; see \Eqref{eq:grad}. 
    \item \textbf{Second Stage}. The GNN teacher is decoupled by its student, and the student is retrained alone to mimic the teacher's predictions via 
    \begin{equation}
        \gL_{\text{KD}}  =  \lambda \sum_{v \in \gV} l_{\text{KL}}(\hat{\vs}_v, \hat{\vt}_v) +  \sum_{v \in \gV^L}   l_{\text{CE}}(\hat{\vs}_v, \vy_v). 
    \end{equation}
\end{itemize}
The second stage is similar to conventional KD, which is why we term this strategy \gradkd (\ngradj+KD). In our case, however, the GNN teacher has been trained with a \graphless student, which allows it to better leverage textual information from unlabeled nodes. %

\subsection{Student Models} \label{sec:gradbert-gnn}

A powerful function $\tau(\cdot)$, that can encode information from both the GNN and the \graphless model, is a key factor in the co-optimization of the teacher and student model. We use LMs, such as \bert~\cite{devlin2019bert}, to instantiate $\tau(\cdot)$, and name our method \gradbert. \gradbert is fine-tuned end-to-end, which is shown to benefit node classification tasks~\cite{ioannidis2022graph-aware-bert}. 

\gradbert requires raw texts on nodes and assumes the graphs have adequate data for training. Although not common in practical scenarios, some graphs might be of small scale and not associated with  raw texts. In such cases, we model $\tau(\cdot)$ by bag-of-words or TF-IDF text vectors (if provided), followed by trainable MLPs, which we name  \gradmlp. 

\section{Experimental Setup} \label{sec:exp-setup}
\subsection{Datasets} 
For \gradbert, we use three widely used node classification benchmarks that provide raw text on nodes, ogbn-arxiv (Arxiv), ogbn-products (Products), and a smaller version of ogbn-papers100M (Papers1.5M)~\cite{hu2020ogb,chien2022giant}. We also use a private heterogeneous graph (Product-Reviews) in Appendix~\ref{app:prod-rev}, and \emph{perturbed} datasets in Appendix~\ref{app-pert}.
For \gradmlp, we use Cora, Citeseer, Pubmed, A-Computer, and A-Photo~\cite{yang2021cpf}, where input features are bag-of-words or TF-IDF word encodings. 
The statistics of these datasets are presented in Appendix~\ref{app:data} and the evaluation metric is accuracy (\%).

\subsection{Implementation Details} 
For our GNN teacher, we use GraphSAGE~\cite{hamilton2017graphsage} for both the \gradbert and the \gradmlp students. To further reduce the training computation cost for \gradbert, we use an 1-layer GraphSAGE and sample $S \in \{8,12\}$ neighbors for \Eqref{eq:sage}, which has a training cost of $\gO(SC)$ per node. 
We initialize \gradbert parameters with  SciBERT~\cite{beltagy2019scibert} for Arxiv and Papers1.5M  that provides better tokenization for scientific texts. Due to computational constraints, we use a smaller LM for Products, i.e.,  DistilBERT~\cite{sanh2019distilbert}, and we  only use \grada for small-scale graphs. 
 Further implementation details can be found in Appendix~\ref{app:impl}. Ablation studies on \grad's framework are conducted in Appendix~\ref{app:hyper}. 

\subsection{Compared Methods}  

We develop the BERT+KD baseline, which employs conventional GNN-to-BERT KD via~\Eqref{eq:kd}. GLNN~\cite{zhang2022glnn} is a GNN-to-MLP distillation approach that does not leverage LMs, to which we refer as MLP+KD baseline. 
We also compare \gradbert with methods that fine-tune LMs on graphs. Graph-Aware BERT (GA-BERT)~\cite{ioannidis2022graph-aware-bert} and BERT-LP~\cite{chien2022giant} fine-tune BERT models to solve the link prediction task. GIANT~\cite{chien2022giant} fine-tunes BERT to solve neighborhood prediction, which is a task similar to link prediction. E2EG~\cite{dinh2022e2eg} trains GIANT for neighborhood prediction and node classification end-to-end. GLEM~\cite{zhao2022glem} utilizes variational inference to jointly optimize the LM and the GNN for node classification. 
For a fair comparison, we have the same GNN architecture (GraphSAGE) among methods that use GNNs (\grad, KD, GA-BERT, and GLEM) during training.

\section{Experimental Results}

In the following experiments, we assess \grad's performance for node classification over textual graphs. In \Secref{sec:exp-graphless}, we compare \gradbert with other methods that do not use graph structure during inference, as well as we assess the performance of \grad's GNN teacher. In \Secref{sec:exp-ind}, we evaluate how well different \grad strategies generalize to inductive (unseen) nodes. Furthermore, we conduct inference time analysis to demonstrate the efficiency advantage of \grad in Appendix~\ref{sec:inf} and provide qualitative examples in Appendix~\ref{sec:exp-fig}.

\subsection{\gradbert Results} \label{sec:exp-graphless}

\begin{table}[tb]
	\centering
	\caption{Performance comparison of \graphless methods.  Bold font denotes the overall best results. We also report the number of trainable parameters of each LM model in millions (\#Params).}
	\label{tab:res}%
	\resizebox{0.85\linewidth}{!}{
	\begin{threeparttable}
		\begin{tabular}{llc|lc|lc}
			\toprule
			 & \multicolumn{2}{c}{\textbf{Arxiv}} & \multicolumn{2}{c}{\textbf{Products}} & \multicolumn{2}{c}{\textbf{Papers1.5M}} \\
    & Acc. (\%) & \#Params & Acc. (\%) & \#Params  & Acc. (\%) & \#Params \\
			\midrule
			
            MLP &  $62.91_{\pm 0.60}$ &\small{No LM} & $61.06_{\pm 0.08}$ &No LM& $47.24_{\pm 0.39}$ &No LM\\
            GLNN & $72.15_{\pm 0.27}$ &No LM& $77.65_{\pm 0.48}$ & No LM& \multicolumn{2}{c}{-}\\
            \midrule 
            BERT &  $72.81_{\pm 0.12}$ &110M& $77.64_{\pm 0.08}$ & 110M& $61.45_{\pm 0.07}$ &110M\\
            BERT-LP &  $67.33_{\pm 0.54}$ &110M& $73.83_{\pm 0.06}$ &110M&  \multicolumn{2}{c}{-}\\
            GIANT  &  $73.06_{\pm 0.11}$ &110M& $80.49_{\pm 0.28}$ &110M & $61.10_{\pm 0.19}$ &110M\\
            E2EG & $73.62 _{\pm 0.14}$& 110M& $81.11_{\pm 0.37}$ &110M & \multicolumn{2}{c}{-}\\
            GLEM &  $74.53_{\pm 0.12}$ &138M& $81.25_{\pm 0.15}$ &138M &\multicolumn{2}{c}{-}\\
			\midrule 
            BERT+KD & $74.39_{\pm 0.32}$ &110M& $\underline{81.91}_{\pm 0.64}$ &66M & $61.85_{\pm 0.04}$ &110M\\
            \midrule
            \gradbert \gradbertemoji{} & & &  & & &\\
            \; \ngradj &  $\underline{74.92}_{\pm 0.16}$ &110M& $81.42_{\pm 0.40}$& 66M& $\underline{63.44}_{ \pm 0.05}$ &110M\\
            \; \ngradkd & $\textbf{75.05}_{\pm 0.11}$ &110M& $\textbf{82.89}_{\pm 0.07}$ &66M& $\textbf{63.60}_{\pm 0.05}$& 110M\\
			\bottomrule
		\end{tabular}%
		\end{threeparttable}
}

\end{table}%

Table~\ref{tab:res} shows performance results for methods that do not use graph structure during inference. Clearly, \gradbert is the method that performs the best.  MLP and GLNN  are methods that do not fine-tune LMs, and thus, perform poorly. GIANT and BERT-LP pretrain BERT to encode structural information, but BERT is fixed for node classification, which may be suboptimal. For example, E2EG, that adapts GIANT for node classification, improves GIANT by 0.61\% points on average. GLEM relies on powerful LMs to alternately fine-tune the teacher and the student models, and thus outperforms previous methods. However, GLEM does not show a clear advantage over our baseline BERT+KD method, while it utilizes 28M or 77M more parameters.
\gradbert outperforms BERT+KD by 0.66\%, 0.98\%, and 1.75\% points for Arxiv, Products, and Papers, respectively, which is a considerable improvement for these large-scale graphs. In Products, \gradkd improves over \gradj by 1.45\% points, and shows that JKD allows the student model to better fit to its teacher's predictions.

\begin{table}[tb]
	\centering
	\caption{GNN performance comparison for different methods that jointly fine-tune LMs for node  classification. The underlying GNN model is GraphSAGE. We also report the number of trainable parameters of each backbone LM model in millions (\#Params).}
	\label{tab:gnn-teacher}%
	\resizebox{0.9\columnwidth}{!}{
	\begin{threeparttable}
		\begin{tabular}{llc|lc|lc}
			\toprule
			& \multicolumn{2}{c}{\textbf{Arxiv}} & \multicolumn{2}{c}{\textbf{Products}} & \multicolumn{2}{c}{\textbf{Papers1.5M}} \\
    & Acc. (\%) & \#Params & Acc. (\%) & \#Params  & Acc. (\%) & \#Params \\
			\midrule
			
            BERT-GNN  &  $74.78_{\pm 0.52}$ & 110M & $82.01_{\pm 0.43}$ & 110M & $65.80_{\pm 0.23}$ & 110M \\
            GA-BERT-GNN & $74.97$ & 110M & $82.35$ & 110M & \multicolumn{2}{c}{-} \\
            GLEM-GNN &  $75.50_{\pm 0.24}$ & 138M & $83.16_{\pm 0.19}$ & 138M  & \multicolumn{2}{c}{-} \\
            \grad-GNN  &  $\textbf{76.42}_{\pm 0.21}$ & 110M & $\textbf{83.34}_{\pm 0.24}$ & 66M &  $\textbf{66.61}_{\pm 0.22}$ & 110M \\
			\bottomrule
		\end{tabular}%
		\end{threeparttable}
}

\end{table}%

Table~\ref{tab:gnn-teacher} evaluates the performance of GNNs combined with LMs (LM-GNN methods). \grad's GNN teacher performs the best, which verifies that \grad's student improves its teacher's effectiveness. BERT-GNN is the baseline LM-GNN with end-to-end training, which performs the worst. GA-BERT enhances BERT with a link prediction task and leads to better results than BERT-GNN. GLEM that leverages unlabeled nodes, as \grad does, performs better than BERT and GA-BERT. However, it does not train the LM-GNN model end-to-end, which limits its performance compared to \grad,  while it requires larger LM models. In Appendix~\ref{app:gnn-full}, we also evaluate \gradbert with state-of-the-art downstream GNNs. 

In transductive settings, \graphless students could perform well by imitating their GNN teachers' predictions without learning effective features (\Secref{sec:mi}). Thus, we evaluate our \grad approach in the inductive setting, where well-trained students are essential to generalize to features of unseen nodes, which are not connected to the existing graph during training. Table~\ref{tab:inductive} shows  \gradbert's performance in inductive settings. \gradbert outperforms the baseline BERT+KD by more than 0.7\% points, on average. Note that MLP+KD  uses a static BERT model and cannot generalize well to unseen nodes. As Table~\ref{tab:inductive} shows, methods that use conventional KD (BERT+KD and \gradkd) have a higher performance drop for inductive nodes. This suggests that conventional KD favors transductive settings while it might be limited for inductive nodes, which is also analyzed in Appendix~\ref{app:coro}.

\begin{table}[tb]
	\centering
	\caption{Performance comparison of different distillation strategies in the inductive setting, in which we hold 50\% of validation and test nodes out of the full graph. \textit{Full} means that methods are trained and evaluated on the full graph, while \textit{ind} shows performance on inductive nodes. $\Delta_{\text{acc}}$ reports the relative performance degradation between the two settings. }
	\label{tab:inductive}%
	\resizebox{0.6\columnwidth}{!}{
	\begin{threeparttable}
		\begin{tabular}{lc|cccc}
			\toprule
              & & MLP+KD & BERT+KD & \multicolumn{2}{c}{\gradbert} \\
              \multicolumn{2}{l|}{Dataset} & & & \ngradj & \ngradkd \\
              \midrule
			 \multirow{3}{*}{\textbf{Arxiv}} & full & 68.26 &  74.43 & 74.45 & 75.15 \\
                            &  \textbf{ind} & 59.34 &  73.45 & \textbf{74.22} & 73.81\\
                            & $\Delta_{\text{acc}}$ & -13.06\% &   -1.32\% & \underline{-0.31}\% & -1.78\%  \\
             \hline
             \multirow{3}{*}{\textbf{Products}} & full & 77.20 &  82.31 & 81.47 & 82.82\\
                            &  \textbf{ind} & 72.93  & 80.94  & 81.24 & \textbf{81.69}\\
                            & $\Delta_{\text{acc}}$ &  -5.53\% &  -1.66\% & \underline{-0.28}\% & -1.36\%  \\
			\bottomrule
		\end{tabular}%
		\end{threeparttable}
}

\end{table}%

\subsection{\gradmlp Results} \label{sec:exp-ind}

\begin{table}[tb]
	\centering
	\caption{Performance comparison of different \grad strategies on transductive (tran) and inductive (ind) nodes. We sample 50\% nodes from the test set as the inductive nodes to evaluate every method. We report mean accuracy over 10 runs.}
	\label{tab:inductive-mlp}%
	\resizebox{0.6\linewidth}{!}{
	\begin{threeparttable}
        \begin{tabular}{lc|cc|cccc}
        \toprule
          & &  GLNN & GLEM & \multicolumn{3}{c}{\gradmlp (ours)}  \\
          \cline{5-7}
         \multicolumn{2}{l|}{Dataset} & (MLP+KD) & (MLP) & \ngradj  & \ngrada & \ngradkd \\
        \midrule
   \multirow{2}{*}{\textbf{Cora}} & tran& $76.72$ &  $76.78$&  $79.39$ &$79.20$  &$79.04$ \\
        & \textbf{ind} & $70.74$ & $70.44$ & $\textbf{73.00}$ & $72.86$ & $72.51$	 \\
        \hline
        \multirow{2}{*}{\textbf{Citeseer}} & tran& $69.55$ &  $69.06$  & $70.33$ & $70.96$  & $68.69$ \\
        & \textbf{ind} & $69.72$ & $69.33$ &   $70.38$ & $\textbf{70.80}$ & $69.28$ \\
        \hline
         \multirow{2}{*}{\textbf{Pubmed}} & tran& $74.00$ &  $76.38$ & $78.13$  & $77.71$ & $77.72$  \\
        & \textbf{ind} & $73.61$ & $75.53$  & $\textbf{77.54}$ & $76.96$ & $76.89$ \\
        \hline
        \multirow{2}{*}{\textbf{A-Comp}} & tran& $82.12$ & $81.15$  & $81.20$  & $82.60$ &   $82.28$ \\
        & \textbf{ind} & $79.11$ & $78.07$ & $78.79$ & $\textbf{80.09}$ &  $79.41$ \\
        \hline
        \multirow{2}{*}{\textbf{A-Photo}} & tran& $92.21$ & $89.83$ & $91.68$ & $91.70$ & $92.36$ \\
        & \textbf{ind} & $89.96$ & $86.49$  & $88.83$ & $89.03$ & $\textbf{90.09}$ 	  \\
        \bottomrule
	\end{tabular}%

	\end{threeparttable}
}
\end{table}%

Table~\ref{tab:inductive-mlp}  shows the performance of different \grad strategies for small-scale graphs, where MLPs uses as the backbone share function $\tau$. Methods that employ KD do not couple the teacher and student models sufficiently, and  thus do not learn as effective node features in these small-scale graphs. GLEM performs the worst as it requires a powerful student model that generates high-quality soft-labels, which is not the case for simple MLPs. \grada performs the best on Citeseer and A-Comp, while \gradj performs the best on Cora and Pubmed. We suspect that, in the second case, node texts are more informative (paper keywords) and thus, coupling the teacher and student models tightly improves the performance.  

Moreover, \Figref{fig:ind-abla} shows that \gradmlp consistently outperforms  MLP+KD with different inductive rates, ranging from 10\% to 75\%. In these small graphs, having an inductive ration close to 90\% results in training graphs with few connections among nodes. Thus, in this case, \gradmlp does not have useful graph information to leverage. 

\section{Related Work}

\textbf{Learning on Textual Graphs}. A common approach for effective learning on textual graphs is to combine LMs~\cite{devlin2019bert,liu2019roberta,raffel2020t5} with GNNs~\cite{kipf2017gcn,velivckovic2017gat,hamilton2017graphsage}. Methods such as~\cite{ioannidis2022graph-aware-bert,chien2022giant,yasunaga2022dragon} focus on pre-training LMs over graph data to preserve node-level structure, while~\cite{yasunaga2022linkbert} focuses on token-level information. Methods such as~\cite{zhao2022glem,dinh2022e2eg,yang2021graphformers,zhang2020graphbert} focus on fine-tuning LMs for solving node classification directly, which is more related to our work. Experiments showed that \grad outperforms these competing approaches. Apart from node classification, leveraging GNNs for learning on textual graphs has been applied to question answering~\cite{mavromatis2022rearev,zhang2022greaselm}, sponsored search~\cite{zhu2021textgnn,li2021adsgnn}, and document classification~\cite{yao2019textgnn1,huang2019textgnn2}.

\textbf{Distillation Approaches}. Distillation approaches mainly focus on model compression (see a survey in~\cite{gou2021knowledge}). It is worth to mention that ~\cite{zhang2019selfdistillation} proposed a self-distillation technique that distills knowledge from deeper layers to shallow ones of the same architecture. Inspired by self-distillation, \gradj self-distills knowledge from deeper layers (GNN) to shallow layers (LM or MLP). However, our work is motivated by capturing node interactions, and not by model compression. Moreover,~\cite{yuan2020revisitingsd} shows a connection between self-distillation and label smoothing, which may be interpreted as a graph-aware label smoothing in our case. 
Closely related to our work, GLNN~\cite{zhang2022glnn} and ColdBrew~\cite{zheng2021coldbrew} propose to distill knowledge from a graph modality to a graph-less modality. However, experiments showed that \grad learns better graph information than GLNN for both inductive and transductive settings, while ColdBrew can only be applied to transductive settings. Graph-regularized MLPs~\cite{zhou2003lapreg,ando2006lapreg,yang2021preg,hu2021graphmlp,dong2022n2n} are also methods that improve MLP performance for node classification. However, their performance is inferior to the one achieved by KD approaches. Other graph-based distillation approaches focus on  distilling large GNNs to smaller GNNs~\cite{yang2020distillinggcn,yan2020tinygnn} or to simpler graph models~\cite{yang2021extract}. Finally, the work in~\cite{deng2021graph-free} distills knowledge from pretrained GNNs and the work in~\cite{xu2020graphsail} applies graph-based distillation for incremental learning in recommender systems.

\begin{figure}[t]
  \centering
    \resizebox{0.8\columnwidth}{!}{\definecolor{col1}{rgb}{0.60, 0.31, 0.64}
\definecolor{col2}{rgb}{0.30, 0.69, 0.29}
\definecolor{col3}{rgb}{0.22, 0.49, 0.72}
\definecolor{col4}{rgb}{0.89, 0.10, 0.11}
\definecolor{col5}{rgb}{1, 1, 0.8}

\begin{tikzpicture}
\tikzstyle{every node}=[font=\Huge]

\begin{axis}[legend style={at={(.9,0.9),anchor=north east}},
             legend style={legend pos=outer north east,},  title=Cora,ylabel={Accuracy (\%)}, every axis plot/.append style={ultra thick}, ymin=60, ymax=75, compat=1.5
]
\addplot[mark=square*,col2] coordinates {
    (10, 72.58)
    (25, 73.66)
    (50, 73.00)
    (75, 68.83)
    (90, 66.44)
};

\addplot[mark=*,col4] coordinates {
    (10,69.58)
    (25, 71.11)
    (50, 70.74)
    (75, 65.54)
    (90, 65.45)
};
\end{axis}

\begin{axis}[legend columns=-1,legend style={at={(1,1.35),anchor=north east,}},  title={Citeseer}, xlabel={ Inductive Node Rate (\%)}, ymin=60, ymax=75,  every axis plot/.append style={ultra thick},  xshift=8cm, compat=1.5
]
\addplot[mark=square*,col2] coordinates {
    (10, 74.70)
    (25, 73.05)
    (50, 73.00)
    (75, 65.30)
    (90, 61.05)
};
\addlegendentry{\gradmlp \; \; } 

\addplot[mark=*,col4] coordinates {
    (10,71.27)
    (25, 69.96)
    (50, 69.72)
    (75, 67.17)
    (90, 64.51)
};
\addlegendentry{MLP+KD} 
\end{axis}

\begin{axis}[legend style={at={(.9,0.9),anchor=north east}},
             legend style={legend pos=outer north east,},  title=Pubmed, every axis plot/.append style={ultra thick}, ymin=65, ymax=80, xshift=16cm
]
\addplot[mark=square*,col2] coordinates {
    (10, 79.01)
    (25, 77.37)
    (50, 77.54)
    (75, 76.12)
    (90, 69.59)
};

\addplot[mark=*,col4] coordinates {
    (10, 77.07)
    (25, 75.22)
    (50, 73.61)
    (75, 71.43)
    (90, 70.56)
};
\end{axis}
\end{tikzpicture}}
    \caption{Performance on inductive nodes w.r.t. inductive node rate (inductive node rate equals to \#ind test nodes / \#total test nodes). \gradmlp consistently outperforms MLP+KD for reasonable inductive rates.}
    \label{fig:ind-abla}
\end{figure}
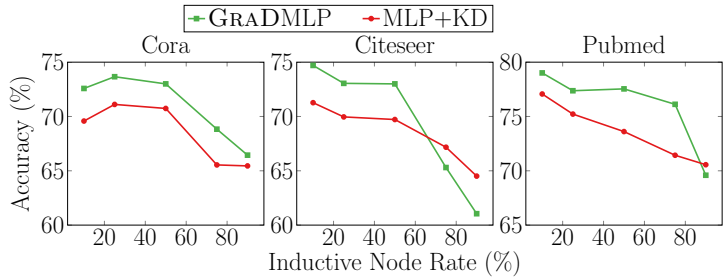
 
\section{Conclusion }
In this paper, we developed a graph-aware distillation approach (\grad) that jointly trains a teacher GNN and a \graphless student via a shared LM. This allows the two models to learn from each other and improve their overall performance.
We have evaluated \grad in eight node classification benchmarks in both transductive and inductive settings, in all of which \grad outperforms conventional knowledge distillation. \grad is a method that achieves a balance among efficiency and effectiveness in textual graphs. 

\section{ Limitations \& Ethical Statement}
\grad relies on informative input node features to learn effective shared LMs (or MLPs) that can generalize to unseen nodes, which is the case in textual graphs. Thus, one limitation is that it is not certain how \grad generalizes to other graphs, e.g., to featureless graphs. Moreover  as a knowledge distillation approach, \grad trades accuracy for computation efficiency and it cannot adapt to dynamic graphs with edge changes the same way as GNN could. 
To overcome biases encoded in the training graph, e.g., standard stereotypes in recommender graphs, \grad needs to be retrained over the new unbiased graph. 

\bibliographystyle{splncs04}
\bibliography{Distillation}

\newpage
\appendix

\begin{table*}[ht]
\centering
\caption{Statistics of the datasets.}
\label{tab:data}
\resizebox{\linewidth}{!}{

\begin{tabular}{l|c|c|c|c||ccccc}
\toprule
Text & \multicolumn{4}{c|}{\textit{raw text}} &  \multicolumn{5}{c}{\textit{bag-of-words / TF-IDF}} \\
Dataset & \textbf{Arxiv}    & \textbf{Products} & \textbf{Papers 1.5M}  & \textbf{\underline{P}roduct-\underline{R}eviews}  & \textbf{Cora} & \textbf{Citeseer} & \textbf{Pubmed} & \textbf{A-Comp} & \textbf{A-Photo}\\
\midrule
\multirow{2}{*}{\# Nodes} & \multirow{2}{*}{169,343} & \multirow{2}{*}{2,449,029} & \multirow{2}{*}{1,546,782} & P: 2,696,062 & \multirow{2}{*}{2,485} & \multirow{2}{*}{2,110}  & \multirow{2}{*}{19,717} &  \multirow{2}{*}{13,381} & \multirow{2}{*}{7,487} \\
  &  &  &  & R: 2,533,129 & & & & & \\
 \# Edges & 1,166,243  & 2,449,029 & 1,546,782 & 2,096,202 & 5,069 & 3,668 & 44,324 &  245,778 & 119,043 \\
 Avg. Degree & 13.7 & 50.5 & 17.7  & 0.7 & 2.0 & 1.74 & 2.3 & 18.4 & 15.9 \\
 \# Classes & 40 & 47 & 172 & 1,810 & 7 & 6 & 3 & 10 & 8 \\
Train / Val / Test  & 54 / 18 / 28 (\%) &  8 / 2 / 90 (\%) &  78 / 8 / 14 (\%) & Variable &  \multicolumn{5}{c}{\#20 per class / \#30 per class / rest }\\

\bottomrule
\end{tabular}

}
\end{table*}

\section{Dataset Statistics} \label{app:data}

The datasets statistics are shown in Table~\ref{app:data}. Arxiv, Papers, Cora, Citeseer, and Pubmed are citations graphs, while Products, A-Computer, and A-Photo are product co-purchase graphs. Product-Reviews is a heterogeneous graph.

\section{Further Implementation Details} \label{app:impl}

For our teacher GNN, we use GraphSAGE~\cite{hamilton2017graphsage} for both \gradbert and \gradmlp. Our code is available at: \url{https://github.com/cmavro/GRAD}.

\subsection{\gradbert}
For \gradbert, training requirements are demanding, so we implement it with DGL's~\cite{wang2019deep}\footnote{\url{https://www.dgl.ai/}} distributed training~\cite{zheng2020distdgl} and perform training on a p4d.24xlarge instance. To further reduce the training computation cost, we use an 1-layer GraphSAGE for \Eqref{eq:sage}, which has a training cost of $\gO(SC)$ per node. We sample $S=8$ neighbors for Arxiv and Papers, and $S=12$ neighbors for Products. During training, we use a sequence length $L=128$, and during inference, we use either set $L=128$  for ablation studies or $L=512$ for SOTA performance.
We initialize \gradbert parameters with  SciBERT~\cite{beltagy2019scibert} for Arxiv and Papers1.5M and with DistilBERT~\cite{sanh2019distilbert} for Products. For Product-Reviews, we pretrain BERT over the given product descriptions.

\gradj's hyper-parameter $\alpha$ is tuned amongst $\{0.2, 0.4\}$ and hyper-parameter $\lambda$ amongst $\{20,50\}$. Due to demanding computational requirements, we do not perform extensive hyper-parameter search. We set the number of epochs to $10$ and evaluate on the validation set every epoch. We perform model selection based on the best validation scores. We perform mini-batch training to maximize GPU utilization, where the batch size is tuned amongst $\{16,32,48\}$. For inference, we use a batch size of 128.  We optimize the model parameters with Adam~\cite{kingma2014adam} and tuning the learning rate for BERT for $\{1e^{-4}, 1e^{-5}, 1e^{-6}\}$ for the GNN amongst $\{1e^{-3}, 1e^{-4}, 1e^{-5}\}$. We set the number of hidden dimensions to $d=128$.  When combining \gradbert with KD (\gradkd), we perform 5 maximum epochs of KD with the same hyper-parameters and perform model selection based on the validation score.

\subsection{\gradmlp}
For a fair comparison to GLNN (MLP+KD), we use the official GLNN implementation and experimental setup\footnote{\url{https://github.com/snap-research/graphless-neural-networks}}. Similar to GLNN, we use a 2-layer GraphSAGE teacher, where $S=5$ neighbor nodes are sampled at each layer.

Hyper-parameter $\alpha$ of \grad objective is tuned amongst $\{0.2, 0.4\}$ and hyper-parameter $\lambda$ amongst $\{20,50\}$. The number of epochs is tuned amongst values of $\{100, 200, 400, 1000, 2000\}$ with early stopping. We use Adam optimizer with learning rate  tuned amongst $\{1e^{-4}, 1e^{-5}\}$, weight decay tuned amongst $\{5e^{-5}, 5e^{-6}\}$ and dropout~\cite{srivastava2014dropout}. We perform model selection based on the validation scores. When combining \gradmlp with KD, we tune $\alpha$ amongst $\{ 0.4, 0.6, 0.8\}$, which gives more importance to the GNN teacher. Hidden dimensions are set to $d=128$, and we have identical MLP architectures with GLNN. We perform full-graph training. For \grada we perform T-S steps at every optimization step of the full-graph training (1 TS cycle for every target node at each epoch).

\section{Inference Timing Analysis} \label{sec:inf}
We provide inference time analysis between \gradbert and GNNs on two different machines (p4d.24xlarge and p3.16xlarge) for ogbn-arxiv. The p4d.24xlarge is estimated to be 2.5x faster than the p3.16xlarge, while it has 2.5x more GPU memory. As \Figref{fig:abla-inf} shows, \gradbert is the method that that best balances efficiency and effectiveness. 
\begin{figure}
  \centering
    \resizebox{\linewidth}{!}{\definecolor{col1}{rgb}{0.60, 0.31, 0.64}
\definecolor{col2}{rgb}{0.30, 0.69, 0.29}
\definecolor{col3}{rgb}{0.22, 0.49, 0.72}
\definecolor{col4}{rgb}{0.89, 0.10, 0.11}
\definecolor{col5}{rgb}{1, 1, 0.8}

\begin{tikzpicture}

\tikzstyle{every node}=[font=\Large]
\begin{axis}[enlargelimits=false, ylabel={\Large Accuracy (\%)} , xlabel={\Large Inference Time (sec.)}, ymin=70, ymax=77,  legend pos=outer north east, title={\Large p4d.24xlarge-machine}, xtick={34, 218, 445, 600}, xticklabels={34, 218, 445, OOM}, xmin=-60, xmax=640]
    \addplot[
        scatter/classes={
        a={mark=*,mark size = 3pt,teal}, 
        b={mark=*,mark size = 3pt,brown}, 
        c={mark=*,mark size = 3pt,violet},
        d={mark=square*,mark size = 4pt,teal},
        e={mark=square*,mark size = 4pt,brown},
        f={mark=square*,mark size = 4pt,violet},
        g={mark=triangle*,mark size = 7pt,teal},
        h={mark=triangle*,mark size = 7pt,brown},
        i={mark=triangle*,mark size = 7pt,violet}},
        scatter, mark=*, only marks, 
        scatter src=explicit symbolic,
        nodes near coords*={\Label},
        mark size=2pt,
        visualization depends on={value \thisrow{label} \as \Label} %
    ] table [meta=class] {
        x y class label
        16 72.24 a \;
        16 70.43 b {\color{brown}\textbf{BERT}}
        218 73.38 c \;
        34 75.05 d \;
        34 71.81 e \;
        445 75.78 f {\color{violet}\textbf{GNN}}
        182 76.22 g {\color{teal}\textbf{GradBERT}}
        182 72.98 h \;
        600 70 i \;
    };

\end{axis}

\begin{axis}[ legend style={legend pos=outer north east,}, enlargelimits=false, xlabel={\Large  Inference Time (sec.)}, ymin=70, ymax=77, xshift=7.5cm, legend pos=outer north east, title={\Large  p3d.16xlarge-machine}, yticklabels={},xtick={0,161,727,1142, 1500}, xticklabels={0,161,727,1142, OOM}, xmin=-150, xmax=1500]
    \addplot[
        scatter/classes={
        a={mark=*,mark size = 3pt,teal}, 
        b={mark=*,mark size = 3pt,brown}, 
        c={mark=*,mark size = 3pt,violet},
        d={mark=square*,mark size = 5pt,teal},
        e={mark=square*,mark size = 5pt,brown},
        f={mark=square*,mark size = 5pt,violet},
        g={mark=triangle*,mark size = 9pt,teal},
        h={mark=triangle*,mark size = 9pt,brown},
        i={mark=triangle*,mark size = 9pt,violet}},
        scatter, mark=*, only marks, 
        scatter src=explicit symbolic,
        nodes near coords*={\Label},
        mark size=2pt,
        visualization depends on={value \thisrow{label} \as \Label} %
    ] table [meta=class] {
        x y class label
        80 72.24 a \;
        80 70.43 b {\color{brown}\textbf{BERT}}
        1142 73.38 c {\color{violet}\textbf{GNN}}
        161 75.05 d {\color{teal}\textbf{GradBERT}}
        161 71.81 e \;
        1500 70 f \;
        727 76.22 g {\color{teal}\textbf{GradBERT}}
        727 72.98 h {\color{brown}\textbf{BERT}}
        1500 70 i \;
    };
    
    \legend{L=64 GradBERT, -- \; BERT , -- \; GNN ,
    L=128 GradBERT, -- \; BERT , -- \; GNN ,
    L=512 GradBERT, -- \; BERT , -- \; GNN }

\end{axis}

\end{tikzpicture}}
    \caption{ \myTag{\gradbert} is on the Pareto front, being the method that best balances efficiency and effectiveness: Accuracy performance for ogbn-arxiv w.r.t. inference time. 
    $L$ is the input sequence length and
    OOM (out of memory) means that the model encountered GPU failure. }
    \label{fig:abla-inf}
\end{figure}
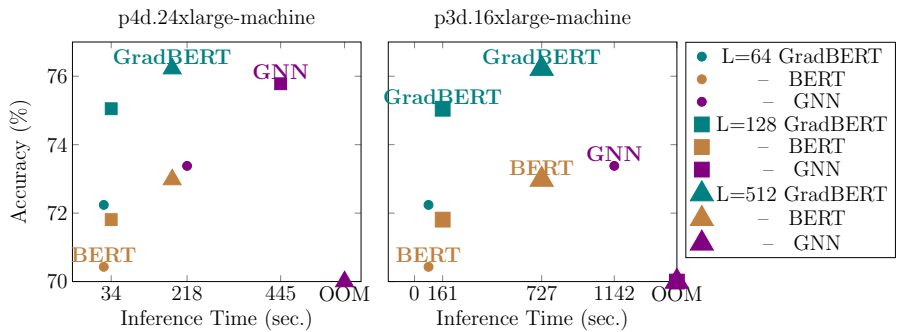

\section{Ablation Studies} \label{app:hyper}

\subsection{Ablation Study on \grad's Loss}

We select \gradj and perform an ablation study on how the distillation term of \Eqref{eq:grad}, which is controlled by hyper-parameter $\lambda$, impacts teacher-student performance. Table~\ref{tab:lamb} verifies that KD loss benefits both the teacher and the student (\grad row). The most significant benefit is observed for the student, which indicates that GNNs distill useful graph-aware information.

\begin{table}
	\centering
	\caption{KD loss effect in \Eqref{eq:grad} for Arxiv dataset.}
	\label{tab:lamb}%
	\resizebox{0.6\columnwidth}{!}{
	\begin{threeparttable}
		\begin{tabular}{lcc}
			\toprule
			 Model & Configuration &  Arxiv accuracy \\ %
			 \midrule
             \multirow{3}{*}{Student} & $\alpha=0, \lambda=0$ & 72.98 \\
			  & $\alpha=0.2, \lambda=0$ & 72.76 \\
			  & \grad & \textbf{74.92} \\
            \midrule
			\multirow{3}{*}{Teacher} & $\alpha=1, \lambda=0$ & 74.78 \\
			  &  $\alpha=0.8, \lambda=0$ & 74.94 \\
			  & \grad & \textbf{75.78} \\
			\bottomrule
		\end{tabular}%
        \end{threeparttable}
        }
\end{table}

\subsection{Hyper-Parameter Sensitivity Analysis}
We continue with more extensive ablation studies for Products dataset (which is the dataset with the largest amount of unlabeled nodes). Table~\ref{tab:lamb2} show the results with $\alpha \in \{0.2, 0.8\}$ and $\lambda \in \{20, 50\}$. In general, due to the excessive amount of unlabeled nodes compared to labeled ones, higher $\lambda$ might overfits to poor GNN predictions. As $\alpha$ controls the importance of each model, $\alpha = 0.2$ leads to a better student, and  $\alpha = 0.8$ leads to a better teacher.

\begin{table}
	\centering
	\caption{KD loss effect in \Eqref{eq:grad} for Arxiv dataset.}
	\label{tab:lamb2}%
	\resizebox{0.6\columnwidth}{!}{
	\begin{threeparttable}
		\begin{tabular}{lcc}
			\toprule
			 Model & Configuration &  Products accuracy \\ %
			 \midrule
             \multirow{4}{*}{Student} & $\alpha=0.2, \lambda=20$ & \textbf{81.42}\\
             & $\alpha=0.8, \lambda=20$ & 80.91  \\
			  & $\alpha=0.2, \lambda=50$ & 81.12 \\
                & $\alpha=0.8, \lambda=50$ & 81.25 \\
            \midrule
			\multirow{4}{*}{Teacher} & $\alpha=0.2, \lambda=20$ & 82.45 \\
             & $\alpha=0.8, \lambda=20$ & \textbf{83.34} \\
			  & $\alpha=0.2, \lambda=50$ & 82.38 \\
                & $\alpha=0.8, \lambda=50$ & 83.12 \\
			\bottomrule
		\end{tabular}%
        \end{threeparttable}
        }
\end{table}

\section{Qualitative Examples} \label{sec:exp-fig}

\begin{figure}[tb]
    \centering
    \includegraphics[width=0.8\linewidth]{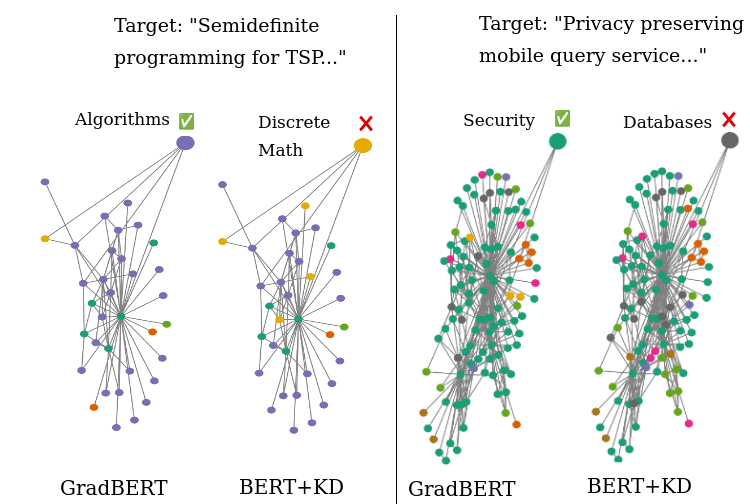}
    \caption{Qualitative examples in which \method outperforms conventional KD. Colors denote node labels. \gradbert leverages label-aware neighbor information to infer the correct label for the target text.}
    \label{fig:abla}
\end{figure}

\Figref{fig:abla} illustrates two  examples from the Arxiv dataset in which \gradbert and BERT+KD have different predictions. In the first example, the phrase ``semidefinite programming" contained in the target text makes BERT+KD incorrectly predict \textit{Discrete Math} as a node label, and in the second, the phrase ``query sevice" makes  BERT+KD incorrectly predict \textit{Databases}. On the other hand, \gradbert better leverages the graph information to decide on what label to predict. In the first example, the majority of neighbors of the target node are \textit{Algorithm} papers and, in the second, \textit{Security} papers. In both cases, \gradbert predicts the correct label.

\section{SOTA GNN Performance} \label{app:gnn-full}

We also assess the quality of \gradbert's learned node representations when used as input \emph{label-aware} node features for other downstream GNNs. Table~\ref{tab:gnn-res-full} shows that \gradbert improves two successful GNNs, RevGAT and SAGN+, by 0.20\%, on average.
Fine-tuning \gradbert performs better than static LMs (OGB-feat) by 3.7\% points, on average. 
For simpler GNNs, i.e., GraphSAGE and its mini-batch version GraphSAINT~\cite{zeng2019graphsaint}, \gradbert performs better than existing approaches by 1\%-1.4\% points. Products dataset has a very small label rate, thus we combine \gradbert with GIANT (GIANT+\gradbert), which is an unsupervised learning method.
\begin{table}[h]
	\centering
	\caption{Performance comparison of GNN methods in transductive settings w.r.t. the input text representations (\textit{LM} column).}
	\label{tab:gnn-res-full}%
	\resizebox{0.7\columnwidth}{!}{
	\begin{threeparttable}
		\begin{tabular}{llll}
			\toprule
			 Dataset & LM & \multicolumn{2}{c}{GNN} \\
			\midrule
			 \multirow{7}{*}{\textbf{Arxiv}}& &  GraphSAGE & RevGAT \\
			 \cline{3-4}
			 & OGB-feat & $71.49_{\pm 0.27}$ & $74.26_{\pm 0.17}$  \\
			 & BERT & $73.42$ &  -\\
		     & GA-BERT & $74.97$ &  - \\
			& GIANT & $74.59_{\pm 0.28}$ & $76.12_{\pm 0.16}$ \\
			& GLEM & $75.50_{\pm 0.24}$ & $76.97_{\pm 0.19}$ \\
			& \gradbert \gradbertemoji{} & $\textbf{76.50}_{\pm 0.07}$& $\textbf{77.21}_{\pm 0.31}$  \\
			\midrule
			\multirow{7}{*}{\textbf{Products}}& &  GraphSAINT & SAGN+ \\
			 \cline{3-4}
			 & OGB-feat& $80.27_{\pm 0.26}$ & $84.85_{\pm 0.10}$ \\
			 & BERT & $81.24$ & - \\
			 & GA-BERT & $82.35$ & - \\
			 & GLEM & $83.16_{\pm 0.19}$ & - \\
		     & GIANT & $84.15_{\pm 0.15}$  & $86.73_{\pm 0.08}$  \\
			 & \; +\gradbert \gradbertemoji{} & $\textbf{85.51}_{\pm 0.51}$ & $\textbf{86.92}_{\pm 0.07}$\\
			\bottomrule
		\end{tabular}%
		\begin{tablenotes}
		\item RevGAT~\cite{li2021revgnn} is a successful GNN for ogbn-arxiv.
            \item SAGN+, i.e., SAGN+SLE+C\&S \cite{sun2021sagn,huang2021cs}, is a successful GNN for ogbn-products.
        \end{tablenotes}
		\end{threeparttable}
}

\end{table}%

\section{Failing KD Case}  \label{app:coro}

\begin{corollary} \label{corollary-app}
Given perfect GNN predictions $\hat{\vt}_v$ obtained via \Eqref{eq:gnn-kd}, the \graphless student can achieve perfect predictions on transductive nodes by setting $X_v := \hat{\vt}_v$ and function $\tau'(\cdot)$ to the identity matrix.
\end{corollary}

\begin{proof}
    The proof is straightforward. Perfect GNN predictions $\hat{\vt}_v$ mean that the predictions obtained by
    \begin{align}
    \begin{split}
    \hat{\vt}_v &= f(\tau(X_u): u \in \gG_v), \\
       \textrm{where } & f, \tau = \argmax_{\tilde{f}, \tilde{\tau}} \sum_{v \in \gV^L} I_{\tilde{f}, \tilde{\tau}}(\vy_v; \gG_v),
       \end{split}
    \end{align}
    maximize the mutual information over the unlabeled nodes
    \begin{align}
     \hat{\vt}_v = \argmax_{\tilde{\vt}_v} \sum_{v \in \gV^U}I(\vy_v, \gG_v) \label{eq:app-proof1}.
    \end{align}

With $X_v := \hat{\vt}_v$, \Eqref{eq:mi2} becomes
    \begin{equation}
         I(\vy_v; \gG_v) =
    I(\vy_v; \hat{\vt}_v) + I(\vy_v;\tilde{\gX}_v | \hat{\vt}_v)  + I(\vy_v ; \gE_v| \gX_v),
    \end{equation}
    which can be further written as
    \begin{equation}
         I(\vy_v; \gG_v) =
    I(\vy_v; \hat{\vt}_v) + I(\vy_v;\tilde{\gX}_v | \hat{\vt}_v)  + I(\vy_v ; \gE_v| (\tilde{\gX}_v, \hat{\vt}_v)),
    \end{equation}
    Because all structural label-aware information is encoded into $\hat{\vt}_v$, $\gE_v$ and $\tilde{\gX}_v$ do not provide additional information for $\vy_v$. Thus, $\vy_v$ is conditionally independent given $\hat{\vt}_v$, which gives $I(\vy_v;\tilde{\gX}_v | \hat{\vt}_v)=0$ and $I(\vy_v ; \gE_v| (\tilde{\gX}_v, \hat{\vt}_v))=0$, due to the mutual information definition. As a result, we have 
    \begin{equation}
         I(\vy_v; \gG_v) = I(\vy_v; \hat{\vt}_v).
    \end{equation}
    The \graphless student solves
    \begin{equation}
        \max_{\tau} I_{\tau}(\vy_v; \tau(\hat{\vt}_v)).
    \end{equation}
    By setting $\tau$ to the identity matrix we have
    \begin{equation}
        I(\vy_v; \hat{\vt}_v),
    \end{equation}
    which is a solution to 
    \begin{equation}
        \max I(\vy_v, \gG_v)
    \end{equation}
    as shown in \Eqref{eq:app-proof1}. Thus, the \graphless student achieves perfect predictions for nodes that $ \hat{\vt}_v$ is given without the need to optimize $ I(\vy_v;  \hat{\vt}_v)$.
\end{proof}
This is, however, a failing case as, for inductive nodes, GNN predictions are not given and the \graphless will only be able to perform random guesses ($\tau$ is an identity matrix without any learned parameters).

\section{Product-Reviews Dataset} \label{app:prod-rev}

One major advantage of \gradbert over simply fine-tuning BERT is that it leverages unlabeled nodes. However, in current datasets, these nodes are directly associated with the hidden labels. On the other hand, Product-Reviews dataset is a heterogeneous graph (Product-Reviews), where there are two different node types; product descriptions and product reviews.  Product-nodes have review-nodes as their side information (product-nodes are not directly connected with each other). In this setting, we test whether \gradbert can leverage this side information (review-nodes) to better predict labels on the target node type (product-nodes). We perform inductive learning, so that BERT and \gradbert use the same product-nodes during training. As Table~\ref{tab:prodrev} shows, \gradbert is able to improve BERT by classifying correctly 1,968-7,926 more products, depending on the data split. This is a considerable improvement although each product has poor graph information with an average degree of 0.7; see Table~\ref{tab:data}.

\begin{table}[h]
	\centering
	\caption{Improvement of \gradbert over BERT for Product-Reviews.  \gradbert leverages the additional side graph information to improve performance on product nodes. }
	\label{tab:prodrev}%
	\resizebox{0.7\columnwidth}{!}{
	\begin{threeparttable}
		\begin{tabular}{lll|c}
			\toprule
			  Split (\%) &   Method & Side Info & Improvement \\ %
			 \midrule
			 \multirow{2}{*}{30/10/60}& BERT &   None  & -- \\ %
			 & \gradbert &   Reviews & \#7,926 Products \\
			  \cline{1-4}
			 \multirow{2}{*}{80/10/10}& BERT &  None  & -- \\ %
		& \gradbert &   Reviews & \#1,968 Products \\
			\bottomrule
		\end{tabular}%
		\end{threeparttable}
}

\end{table}%

\section{Perturbed Node Texts} \label{app-pert}

We study the impact of jointly learning the GNN teacher and \graphless student in a case where the graph structure is important while learning $\tau(\cdot)$. Following our insights from \Secref{sec:mi} that KD may simply copy GNN predictions, we develop \emph{perturbed} datasets in which node texts cannot be easily be mapped to their soft-labels.  Specifically for each node of ogbn-arxiv and ogbn-products graphs, we keep half of its text the same as its original text, but replace the other half by some  noisy text, that is sampled from a random node. For methods that do not use the graph during inference, it is very challenging to discriminate whether useful information comes from the original or and the noisy text parts. For methods that use structure, i.e., GNNs, the majority of the node's neighbors will contain similar semantics with the original text, making the task much easier.

Table~\ref{tab:pert} shows \gradbert's performance with different $\alpha$ values, which control the importance of the GNN term in ~\Eqref{eq:grad}. As we increase $\alpha$, \gradbert performs better, which suggests that \gradbert learns more useful structural information. On the other hand, BERT+KD performs poorly on this task, as it does not use any graph structure while learning $\tau(\cdot)$. We hypothesize this happens because BERT+KD cannot learn an accurate mapping  $\tau(X_v) \xrightarrow{} \hat{\vt}_v$ from text to labels, as analyzed in Section~\ref{sec:mi}.

\begin{table}[h]
	\centering
	\caption{Performance comparison of \gradbert for perturbed datasets w.r.t. different $\alpha$ values (graph importance in \Eqref{eq:grad}).}
	\label{tab:pert}%
	\resizebox{0.7\columnwidth}{!}{
	\begin{threeparttable}
		\begin{tabular}{lcc}
			\toprule
			  &  Arxiv-Perturbed &  Products-Perturbed \\ %
			  & Val / Test & Val / Test \\
			 \midrule
             \method & &  \\
			 \; with $\alpha = 0.8$ & \textbf{49.04} / \textbf{47.35} & \textbf{87.06} / \textbf{63.21} \\
			 \; with $\alpha = 0.2$ & 47.32 / 45.11 & 83.45 / 61.14 \\
            \midrule
			 BERT+KD  & 41.75 / 38.66 & 71.61 / 62.63 \\
			\bottomrule
		\end{tabular}%
		\begin{tablenotes}
		\item BERT+KD is equivalent to setting $\alpha \to 0$.
        \end{tablenotes}
		\end{threeparttable}
}

\end{table}%

\end{document}